\documentclass{interact}
\usepackage{epstopdf}
\usepackage[caption=false]{subfig}
\usepackage[numbers,sort&compress]{natbib}
\usepackage{float}
\usepackage{eurosym}
\usepackage{latexsym}
\usepackage{amssymb,amsmath,amsfonts}
\usepackage[colorlinks=true]{hyperref}
\usepackage{graphicx}
\usepackage{epsfig}
\usepackage{subcaption}

\bibpunct[, ]{[}{]}{,}{n}{,}{,}
\renewcommand\bibfont{\fontsize{10}{12}\selectfont}
\makeatletter
\def\NAT@def@citea{\def\@citea{\NAT@separator}}
\makeatother
\theoremstyle{plain}
\newtheorem{theorem}{Theorem}[section]
\newtheorem{lemma}[theorem]{Lemma}

\newtheorem{algorithm}[theorem]{Algorithm}
\theoremstyle{definition}

\theoremstyle{remark}

\begin{document}

\title{A Double Inertial Forward-Backward Splitting Algorithm With Applications to Regression and Classification Problems}
\author{
\name{\.{I}rfan I\c{S}IK\textsuperscript{a} \thanks{CONTACT \.{I}rfan I\c{S}IK. Email:
irfan.isik11@erzurum.edu.tr}, Ibrahim KARAHAN\textsuperscript{b} and Okan
ERKAYMAZ\textsuperscript{c}}
\affil{\textsuperscript{a,b}Department of Mathematics, Faculty of Science,
Erzurum Technical University, Erzurum, 25100, T\"{u}rkiye;
\textsuperscript{c}Department of Computer Engineering, Naval Academy National Defense University, Istanbul, T\"{u}rkiye.}}
\maketitle

\begin{abstract}
This paper presents an improved forward-backward splitting algorithm with two inertial parameters. It aims to find a point in the real Hilbert space at which the sum of a co-coercive operator and a maximal monotone operator vanishes. Under standard assumptions, our proposed algorithm demonstrates weak convergence. We present numerous experimental results to demonstrate the behavior of the developed algorithm by comparing it with existing algorithms in the literature for regression and data classification problems. Furthermore, these implementations suggest our proposed algorithm yields superior outcomes when benchmarked against other relevant algorithms in existing literature.
\end{abstract}


\begin{keywords}
Double inertial, monotone inclusion problem, regression, classification, weak convergence
\end{keywords}

\section{Introduction}

\bigskip Let assume that $H$ is a real Hilbert Space, $A:H\rightarrow H$ is
a co-coercive operator and $B:H\rightarrow 2^{H}$ is a maximal monotone
operator. In this study, we consider the monotone inclusion problem, which is to find $z\in H$
such that%
\begin{equation}
0\in \left( Az+Bz\right) \text{.}  \label{1.1}
\end{equation}%
Monotone inclusion problems serve as a fundamental tool in addressing a wide
range of problems on machine-learning, linear inverse problems, image
processing, variational inequality problems, convex minimization problems, convex-concave saddle point problems and equilibrium problems, see
e.g. \cite{4, 5, 6, 7, 8, 31}. More specifically, it finds utility across
various scientific domains, including but not limited to statistical
regression, machine learning, image and signal processing see, e.g. \cite%
{10, 11, 12, 13, 29, 42, 43}. Recognizing its significance, numerous authors
have explored suitable algorithms to find solutions for it. The most popular
algorithm to solve (\ref{1.1}) is the forward-backward splitting algorithm
given by Lions and Mercier \cite{35}:%
\begin{equation}
x_{n+1}=\left( I+\lambda _{n}A\right) ^{-1}\left( I-\lambda _{n}B\right)
x_{n}\text{ \ for all }n\in
\mathbb{N}
\label{1.2}
\end{equation}%
where $A$ and $B$ are monotone operators and $\lambda _{n}$ is step size. In
2000, Tseng \cite{31} introduced a two-step strategy as part of the
forward-backward algorithm, necessitating adjustments to equation (\ref{1.1}%
). Tseng's approach is delineated by the following way with $\rho ,k\in (0,1)$, $%
\alpha >0$\ and $\theta _{1}\in H$,
\begin{eqnarray*}
y_{n} &=&(I+\gamma _{n}A)^{-1}(I-\gamma _{n}B)x_{n} \\
x_{n+1} &=&x_{n}-\gamma _{n}(Bx_{n}-By_{n})
\end{eqnarray*}%
where $\gamma _{n}=\alpha k^{m_{n}}$\ and
$m_{n}$\ signify the smallest non-negative integer that satisfies the following condition:%
\begin{equation}
\gamma _{n}\left\Vert \nabla f(\theta _{n+1})-\nabla f(\theta
_{n})\right\Vert \leq \rho \left\Vert \theta _{n+1}-\theta _{n}\right\Vert
\text{.}
\end{equation}
It has been demonstrated that $\left\{ x _{n}\right\} $\ converges weakly to a solution of $\left( A+B\right) ^{-1}(0)$. This algorithm, called forward--backward--forward, is one of the most basic algorithms. In
2009, Beck and Teboulle defined the Fast Iterative Shrinkage-Thresholding
Algorithm (FISTA) as follows: Let $\alpha _{1}=1$ and $\rho _{0}=\rho
_{1}\in H$,%
\begin{equation}
y_{n}=x_{n}+\beta _{n}(x_{n}-x_{n-1})\text{,}\\
x_{n+1}=\left( I+\frac{1}{L}\partial f\right) ^{-1}(I-\frac{1}{L}\nabla
f)(y_{n})\text{,}
\end{equation}
where $\beta _{n}=\frac{t_{n}-1}{t_{n+1}}$, $t_{n+1}=\frac{1+\sqrt{%
1+4t_{n}^{2}}}{2}$, and $L$ is Lipschitz constant of $\nabla f$. They demonstrated
the convergence rate and used it for image reconstruction.

Following these studies, new algorithms were defined in order to increase the running speed
of existing algorithms. Adding an inertial extropolation step to algorithms, which was was first suggested by Polyak \cite {21} in 1964, is one way to increase the speed of algorithms. It was later used by Nesterov \cite{22} in 1983 as an inertial step to solve minimization problems. In the past, many researchers have
included an inertial extrapolation step in their algorithms, see, e.g., \cite{23,24,25,17,18,19,40}. Nevertheless, subsequent to these favorable advancements, Poon and Liang~\cite{38,37} highlighted certain drawbacks associated with optimization methods that incorporate only a single-step inertial extrapolation. Although the inertial term has often been employed to accelerate the convergence of algorithms, it is demonstrated in [\cite{37}, Section 4] that the Douglas–Rachford splitting scheme without any inertial component outperformed its counterpart that incorporates a single-step inertial extrapolation. However, Polyak \cite{39}, who argues that increasing the number of inertial steps in the algorithms will increase the current convergence speed of the algorithms,
has not established both the speedup and convergence states of multi-step
inertial algorithms.

In 2018, Dong et al. \cite{32} presented the double inertial Mann algorithm
in their study. In this study, they also provided the convergence of their
suggested algorithm under some mild conditions. Let $\left\{ \alpha _{n}\right\} \subset \left[ 0,\alpha \right] $ and $\left\{
\beta _{n}\right\} \subset \left[ 0,\beta \right] $ are nondecreasing with $%
\alpha _{1}=\beta _{1}=0$ and $\alpha ,\beta \in \left[ 0,1\right) $. In
addition $T$ be a nonexpansive mapping. The algorithm given by Dong et al. is defined as follows:
\begin{eqnarray*}
z_{n} &=&v_{n}+\alpha _{n}\left( v_{n}-v_{n-1}\right) \\
y_{n} &=&v_{n}+\beta _{n}\left( v_{n}-v_{n-1}\right) \\
v_{n+1} &=&\left( 1-\gamma _{n}\right) z_{n}+\gamma _{n}T\left( y_{n}\right).
\end{eqnarray*}%

Not long after, in 2022 Iyiola and Shehu \cite{33} introduced the new
two-step inertial Douglas-Rachford splitting method in their work. This
approach essentially integrates the two-step inertial proximal point
algorithm with the Douglas-Rachford splitting algorithm, (see \cite{34,35}).
Let given $0\leq \alpha <\frac{1}{3}$ and $\frac{3\alpha -1}{3+4\alpha }%
<\beta \leq 0$. Let $\gamma _{0}=\theta _{0}=\theta _{-1}\in H$, $\rho >0$
and $n>0$. The two step inertial Douglas-Rachford splitting method is introduced by the following way:
\begin{eqnarray*}
\theta _{n+1} &=&J_{n}^{A}\left( 2J_{n}^{B}-I\right) \left( \gamma
_{n}\right) +\left( I-J_{n}^{B}\right) \left( \gamma _{n}\right) \\
\gamma _{n+1} &=&\theta _{n+1}+\alpha (\theta _{n+1}-\theta _{n})+\beta
(\theta _{n}-\theta _{n-1})
\end{eqnarray*}%
where $A,B:H\rightarrow 2^{H}$ be maximal monotone operator on $H$.

In 2023 Suantai et al. \cite{36}, inspired by previous work on solving (\ref%
{1.1}), proposed the following Double inertial forward--backward--forward
algorithm for convex minimization problems. Let us denote $\Gamma $ by $\left( A+B\right) ^{-1}\left( 0\right) \neq
\emptyset $ and let $v_{1}>0$, $\rho \in (0,1)$,
$\left\{ \alpha _{n}\right\} $ and $\left\{ \beta _{n}\right\} $ real
positive sequences. Again let $\theta _{0}$, $\theta _{1}\in H$ and $\left\{
\lambda _{n}\right\} $ be nonnegative sequence such that $\sum_{n=1}^{\infty
}\lambda _{n}<\infty $. Given $\theta _{n}$, $\theta _{n-1}$ and compute%
\begin{eqnarray*}
p_{n} &=&\theta _{n}+\alpha _{n}(\theta _{n}-\theta _{n-1})\text{,} \\
r_{n} &=&p_{n}+\beta _{n}(p_{n}-p_{n-1})\text{,} \\
s_{n} &=&(I+v_{n}B)^{-1}(I-v_{n}A)r_{n}\text{.}
\end{eqnarray*}%
If $s_{n}=r_{n}$ then stop and $s_{n}\in \Gamma $. Else%
\begin{equation}
\theta _{n+1}=s_{n}-v_{n}(As_{n}-Ar_{n})\text{.}
\end{equation}
Here%
\begin{equation}
v_{n+1}=\left\{
\begin{array}{c}
\min \left\{ \frac{\rho \left\Vert r_{n}-s_{n}\right\Vert }{\left\Vert
Ar_{n}-As_{n}\right\Vert },v_{n}+\lambda _{n}\right\} \text{ if }\left\Vert
Ar_{n}-As_{n}\right\Vert \neq 0\text{;} \\
v_{n}+\lambda _{n}\text{ otherwise}%
\end{array}%
\right.
\end{equation}
and the algorithm is updated again from the beginning. For Problem (\ref{1.1}), it is shown in [\cite{37}, Chapter 4] that the Douglas–Rachford splitting method equipped with a two-step inertial extrapolation exhibits faster convergence than its variant utilizing a single-step inertial scheme.

Inspired and motivated by the above studies, our paper offers the following significant contributions: We propose a novel splitting method, namely a forward-backward scheme enhanced with a two-step inertial extrapolation technique, and provide a comprehensive investigation of its weak convergence properties. We implement comprehensive
computational experiments to illustrate the efficient of the proposed
algorithm and also exhibit its superior performance in comparison to the certain
algorithms existed in the literature.

\section{Preliminaries}

In this section, we give the basic concepts and definitions which we need to prove the main results. We use the notations \( p_k \rightharpoonup p \) and \( p_k \to p \) in \( H \) to indicate weak and strong convergence of the sequence \( \{p_k\} \subset H \) to \( p \in H \), respectively, as \( k \to \infty \).
For all \( p, y \in H \) and any \( \alpha \in \mathbb{R} \), the following identities and inequality hold:
\[
\begin{array}{ll}
\text{i)} & \|p + y\|^2 \leq \|p\|^2 + 2\langle y, p + y \rangle \\
\text{ii)} & \|p + y\|^2 = \|p\|^2 + 2\langle p, y \rangle + \|y\|^2 \\
\text{iii)} & \|\alpha p + (1 - \alpha)y\|^2 = \alpha \|p\|^2 + (1 - \alpha)\|y\|^2 - \alpha(1 - \alpha)\|p - y\|^2.
\end{array}
\]

Let \( \mathcal{T}: H \rightarrow H \) be an operator on a real Hilbert space \( H \). The mapping \( \mathcal{T} \) is said to satisfy the following properties:

\begin{description}

\item[(i)] \( \mathcal{T} \) is called \emph{nonexpansive} if
\begin{equation}
\|\mathcal{T}p - \mathcal{T}y\| \leq \|p - y\|, \quad \forall p, y \in H. \label{def:nonexpansive}
\end{equation}

\item[(ii)] \( \mathcal{T} \) is said to be \emph{firmly nonexpansive} provided that
\begin{equation}
\|\mathcal{T}p - \mathcal{T}y\|^2 \leq \|p - y\|^2 - \|(I - \mathcal{T})p - (I - \mathcal{T})y\|^2, \quad \forall p, y \in H, \label{def:firm1}
\end{equation}
which is equivalent to
\begin{equation}
\|\mathcal{T}p - \mathcal{T}y\|^2 \leq \langle p - y, \mathcal{T}p - \mathcal{T}y \rangle. \label{def:firm2}
\end{equation}

\item[(iii)] For a given parameter \( \kappa \in (0,1) \), the operator \( \mathcal{T} \) is said to be \emph{\( \kappa \)-averaged} if there exists a nonexpansive mapping \( S: H \to H \) such that
\begin{equation}
\mathcal{T} = (1 - \kappa)I + \kappa S. \label{def:avg_def}
\end{equation}
Equivalently, the following inequality holds:
\begin{equation}
\|\mathcal{T}p - \mathcal{T}y\|^2 \leq \|p - y\|^2 - \frac{1 - \kappa}{\kappa} \|(I - \mathcal{T})p - (I - \mathcal{T})y\|^2, \quad \forall p, y \in H. \label{def:avg_ineq}
\end{equation}

\item[(iv)] The mapping \( \mathcal{T} \) is said to be \emph{\( L \)-Lipschitz continuous} for some constant \( L > 0 \) if
\begin{equation}
\|\mathcal{T}p - \mathcal{T}y\| \leq L \|p - y\|, \quad \forall p, y \in H. \label{def:lipschitz}
\end{equation}

\item[(v)] The operator \( \mathcal{T} \) is \emph{monotone} if it satisfies
\begin{equation}
\langle \mathcal{T}p - \mathcal{T}y, p - y \rangle \geq 0, \quad \forall p, y \in H. \label{def:monotone}
\end{equation}

\item[(vi)] \( \mathcal{T} \) is said to be \emph{\( \eta \)-strongly monotone} for some constant \( \eta > 0 \) if
\begin{equation}
\langle \mathcal{T}p - \mathcal{T}y, p - y \rangle \geq \eta \|p - y\|^2, \quad \forall p, y \in H. \label{def:strong_monotone}
\end{equation}

\item[(vii)] The mapping \( \mathcal{T} \) is \emph{\( \alpha \)-co-coercive} (also known as \emph{\( \alpha \)-inverse strongly monotone}) if there exists a constant \( \alpha > 0 \) such that
\begin{equation}
\langle \mathcal{T}p - \mathcal{T}y, p - y \rangle \geq \alpha \|\mathcal{T}p - \mathcal{T}y\|^2, \quad \forall p, y \in H. \label{def:cocoercive}
\end{equation}

\end{description}

\noindent A set-valued operator \( B: H \rightarrow 2^H \) is called \emph{monotone} if it satisfies
\[
\langle p - y, f - g \rangle \geq 0 \quad \text{for all } p, y \in H,\ f \in Bp,\ g \in By.
\]
The \emph{graph} of \( B \), denoted by \( \operatorname{Gr}(B) \), is defined as
\[
\operatorname{Gr}(B) := \{ (p, f) \in H \times H : f \in Bp \}.
\]
The operator \( B \) is said to be \emph{maximal monotone} if its graph is not properly contained in the graph of any other monotone operator.

\section{Main Results}

This section of the paper addresses a forward-backward splitting approach incorporating two inertial terms, aiming to identify a zero point of the sum of a maximally monotone operator and a co-coercive operator within a real Hilbert space. We then establish results concerning the weak convergence behavior of the sequence generated by the proposed method. The proposed method, denoted as Algorithm 1*, is detailed below:

\begin{algorithm}
\label{a1}(Algorithm 1*)%

\textbf{Step 1. }Select parameters $\delta \leq 0$, $\lambda \in (0,2\alpha)$, $\vartheta \in [0,1)$, and a sequence $(\mathcal{E}_k) \subset (0,1)$. Choose initial points $p_{-1}$, $p_0$, $p_1 \in H$ arbitrarily, and initialize $k = 1$.

\textbf{Step 2. }Update the iterates using the rules:
\begin{eqnarray*}
w_k &=& \vartheta(p_k - p_{k-1}) + \delta(p_{k-1} - p_{k-2})+p_k, \\
p_{k+1} &=& (1 - \mathcal{E}_k) w_k + \mathcal{E}_k \left( J_\lambda^B(w_k - \lambda A w_k) \right),
\end{eqnarray*}
where $J_\lambda^B = (I + \lambda B)^{-1}$ denotes the resolvent of the operator $B$.

\textbf{Step 3. }Increment $k$ by $1$ and return to Step 2.
\end{algorithm}

In the following section, we assume that certain standard conditions hold.

\textbf{Assumption A}

\begin{description}
\item[(i)] The operator $A: H \rightarrow H$ is assumed to be $\alpha$-cocoercive, and $B: H \rightarrow 2^{H}$ is a maximally monotone operator.

\item[(ii)] The set of zeros of the sum, denoted by $(A+B)^{-1}(0)$, corresponding to Problem (\ref{1.1}), is assumed to be nonempty.
\end{description}

We also require that the inertial parameters $\vartheta$, $\delta$, and the sequence $\{\mathcal{E}_k\}$ satisfy the conditions specified below.

\textbf{Assumption B}

\begin{description}
\item[(i)] There exist constants $0 < \mathcal{E}_1 < \mathcal{E}_k < \mathcal{E}_2 < 1$ for each $1 \leq k \leq 2$.

\item[(ii)] The parameter $\vartheta$ satisfies the inequality
\[
0 \leq \vartheta < \min \left\{ \frac{1}{3}, \frac{\mathcal{E}_1(1 - \kappa)}{\mathcal{E}_1(1 - \kappa) + 2\kappa} \right\}, \quad \text{where } \kappa := \frac{2\alpha}{4\alpha - \lambda}.
\]

\item[(iii)] The scalar $\delta \leq 0$ is chosen such that
\[
\max \left\{ -\frac{\mathcal{E}_1(1 - \kappa)(1 - \vartheta) - 2\kappa \vartheta}{\mathcal{E}_1(1 - \kappa)}, \frac{\vartheta(1 + \vartheta) - \mathcal{E}_1\left( \frac{1 - \kappa}{\kappa} \right)(1 - \vartheta)^2}{(1 + \vartheta)\left(1 + \mathcal{E}_1 \left( \frac{1 - \kappa}{\kappa} \right) \right)} \right\} < \delta,
\]
and in addition,
\[
\kappa \vartheta (1 + \vartheta) - \mathcal{E}_1(1 - \vartheta)^2(1 - \kappa) < \kappa \delta (1+2\vartheta - \delta) + 2\mathcal{E}_1(1 - \kappa)(1 + \vartheta)\delta + \mathcal{E}_1(1 - \kappa)\delta^2.
\]
\end{description}

\begin{lemma}
\label{Lem3.1} 
Under Assumptions A and B, the sequence $\{p_k\}$ produced by Algorithm 1* remains bounded.
\end{lemma}

\begin{proof}
Observe that, according to [\cite{2}, Theorem 7], the operator $J_{\lambda}^{B}(I-\lambda A)$ is $\frac{2\alpha}{4\alpha - \lambda}$-averaged. As a result, Algorithm 1* can be equivalently expressed through a fixed point iteration scheme described by:
\begin{eqnarray}
w_k &=& \vartheta (p_k - p_{k-1}) + \delta (p_{k-1} - p_{k-2}) + p_k, \\
p_{k+1} &=& (1 - \mathcal{E}_k)w_k + \mathcal{E}_k \mathcal{T} w_k,
\label{3.1}
\end{eqnarray}
in here $\mathcal{T} := J_{\lambda}^{B}(I - \lambda A)$. Assume that $p^* \in F(\mathcal{T}) = (A+B)^{-1}(0)$. Based on this, we obtain
\begin{eqnarray*}
w_k - p^* &=& p_k + \vartheta(p_k - p_{k-1}) + \delta(p_{k-1} - p_{k-2}) - p^* \\
&=& (1 + \vartheta)(p_k - p^*) - (\vartheta - \delta)(p_{k-1} - p^*) - \delta(p_{k-2} - p^*).
\end{eqnarray*}

Consequently, we have%
\begin{eqnarray}
\left\Vert w_{k}-p^{\ast }\right\Vert ^{2} &=&\left\Vert \left( 1+\vartheta
\right) \left( p_{k}-p^{\ast }\right) -\left( \vartheta -\delta \right) \left(
p_{k-1}-p^{\ast }\right) -\delta \left( p_{k-2}-p^{\ast }\right) \right\Vert
^{2}  \nonumber \\
&=&\left( 1+\vartheta \right) \left\Vert p_{k}-p^{\ast }\right\Vert ^{2}-\left(
\vartheta -\delta \right) \left\Vert p_{k-1}-p^{\ast }\right\Vert ^{2}-\delta
\left\Vert p_{k-2}-p^{\ast }\right\Vert ^{2}  \nonumber \\
&&+\left( 1+\vartheta \right) \left( \vartheta -\delta \right) \left\Vert
p_{k}-p_{k-1}\right\Vert ^{2}+\delta \left( 1-\vartheta \right) \left\Vert
p_{k}-p_{k-2}\right\Vert ^{2}  \nonumber \\
&&-\delta \left( \vartheta -\delta \right) \left\Vert
p_{k-1}-p_{k-2}\right\Vert ^{2}\text{.}  \label{3.2}
\end{eqnarray}%
Observe that%
\begin{eqnarray*}
2\vartheta \left\langle p_{k+1}-p_{k},p_{k}-p_{k-1}\right\rangle
&=&2\left\langle \vartheta \left( p_{k+1}-p_{k}\right)
,p_{k}-p_{k-1}\right\rangle \\
&\leq &2\left\vert \vartheta \right\vert \left\Vert p_{k+1}-p_{k}\right\Vert
\left\Vert p_{k}-p_{k-1}\right\Vert \\
&=&2\vartheta \left\Vert p_{k+1}-p_{k}\right\Vert \left\Vert
p_{k}-p_{k-1}\right\Vert \text{,}
\end{eqnarray*}%
and so%
\begin{equation}
-2\vartheta \left\langle p_{k+1}-p_{k},p_{k}-p_{k-1}\right\rangle \geq -2\vartheta
\left\Vert p_{k+1}-p_{k}\right\Vert \left\Vert p_{k}-p_{k-1}\right\Vert
\text{.}  \label{3.3}
\end{equation}%
Also, we can write%
\begin{eqnarray*}
2\delta \left\langle p_{k+1}-p_{k},p_{k-1}-p_{k-2}\right\rangle
&=&2\left\langle \delta \left( p_{k+1}-p_{k}\right)
,p_{k-1}-p_{k-2}\right\rangle \\
&\leq &2\left\vert \delta \right\vert \left\Vert p_{k+1}-p_{k}\right\Vert
\left\Vert p_{k-1}-p_{k-2}\right\Vert
\end{eqnarray*}%
which implies that%
\begin{equation}
-2\delta \left\langle p_{k+1}-p_{k},p_{k-1}-p_{k-2}\right\rangle \geq
2\left\vert \delta \right\vert \left\Vert p_{k+1}-p_{k}\right\Vert
\left\Vert p_{k-1}-p_{k-2}\right\Vert \text{.}  \label{3.4}
\end{equation}%
Similarly, we note that%
\begin{eqnarray*}
2\delta \vartheta \left\langle p_{k-1}-p_{k},p_{k-1}-p_{k-2}\right\rangle
&=&2\left\langle \delta \vartheta \left( p_{k-1}-p_{k}\right)
,p_{k-1}-p_{k-2}\right\rangle \\
&\leq &2\left\vert \delta \right\vert \vartheta \left\Vert
p_{k-1}-p_{k}\right\Vert \left\Vert p_{k-1}-p_{k-2}\right\Vert \\
&=&2\left\vert \delta \right\vert \vartheta \left\Vert p_{k}-p_{k-1}\right\Vert
\left\Vert p_{k-1}-p_{k-2}\right\Vert \text{,}
\end{eqnarray*}%
and thus,%
\begin{eqnarray}
2\delta \vartheta \left\langle p_{k}-p_{k-1},p_{k-1}-p_{k-2}\right\rangle
&=&-2\delta \vartheta \left\langle p_{k-1}-p_{k},p_{k-1}-p_{k-2}\right\rangle
\nonumber \\
&\geq &-2\left\vert \delta \right\vert \vartheta \left\Vert
p_{k}-p_{k-1}\right\Vert \left\Vert p_{k-1}-p_{k-2}\right\Vert \text{.}
\label{3.5}
\end{eqnarray}%
By (\ref{3.3}),(\ref{3.4}) and (\ref{3.5}), we obtain%
\begin{eqnarray}
\left\Vert p_{k+1}-w_{k}\right\Vert ^{2} &=&\left\Vert p_{k+1}-\left(
p_{k}+\vartheta \left( p_{k}-p_{k-1}\right) +\delta \left(
p_{k-1}-p_{k-2}\right) \right) \right\Vert ^{2}  \nonumber \\
&=&\left\Vert p_{k+1}-p_{k}\right\Vert ^{2}-2\vartheta \left\langle
p_{k+1}-p_{k},p_{k}-p_{k-1}\right\rangle  \nonumber \\
&&-2\delta \left\langle p_{k+1}-p_{k},p_{k-1}-p_{k-2}\right\rangle +\vartheta
^{2}\left\Vert p_{k}-p_{k-1}\right\Vert ^{2}  \nonumber \\
&&-2\delta \vartheta \left\langle p_{k}-p_{k-1},p_{k-1}-p_{k-2}\right\rangle
+\delta ^{2}\left\Vert p_{k-1}-p_{k-2}\right\Vert ^{2}  \nonumber \\
&\geq &\left\Vert p_{k+1}-p_{k}\right\Vert ^{2}-2\vartheta \left\Vert
p_{k+1}-p_{k}\right\Vert \left\Vert p_{k}-p_{k-1}\right\Vert  \nonumber \\
&&-2\left\vert \delta \right\vert \left\Vert p_{k+1}-p_{k}\right\Vert
\left\Vert p_{k-1}-p_{k-2}\right\Vert +\vartheta ^{2}\left\Vert
p_{k}-p_{k-1}\right\Vert ^{2}  \nonumber \\
&&-2\left\vert \delta \right\vert \vartheta \left\Vert p_{k}-p_{k-1}\right\Vert
\left\Vert p_{k-1}-p_{k-2}\right\Vert +\delta ^{2}\left\Vert
p_{k-1}-p_{k-2}\right\Vert ^{2}  \nonumber \\
&\geq &\left\Vert p_{k+1}-p_{k}\right\Vert ^{2}-\vartheta \left\Vert
p_{k+1}-p_{k}\right\Vert ^{2}-\vartheta \left\Vert p_{k}-p_{k-1}\right\Vert ^{2}
\nonumber \\
&&-\left\vert \delta \right\vert \left\Vert p_{k+1}-p_{k}\right\Vert
^{2}-\left\vert \delta \right\vert \left\Vert p_{k-1}-p_{k-2}\right\Vert
^{2}\vartheta ^{2}\left\Vert p_{k}-p_{k-1}\right\Vert ^{2}  \nonumber \\
&&-\left\vert \delta \right\vert \vartheta \left\Vert p_{k}-p_{k-1}\right\Vert
^{2}-\left\vert \delta \right\vert \vartheta \left\Vert
p_{k-1}-p_{k-2}\right\Vert ^{2}+\delta ^{2}\left\Vert
p_{k-1}-p_{k-2}\right\Vert ^{2}  \nonumber \\
&\geq &\left( 1-\left\vert \delta \right\vert -\vartheta \right) \left\Vert
p_{k+1}-p_{k}\right\Vert ^{2}-\left( \vartheta ^{2}-\vartheta -\left\vert \delta
\right\vert \vartheta \right) \left\Vert p_{k}-p_{k-1}\right\Vert ^{2}
\nonumber \\
&&+\left( \delta ^{2}-\left\vert \delta \right\vert -\left\vert \delta
\right\vert \vartheta \right) \left\Vert p_{k-1}-p_{k-2}\right\Vert ^{2} \text{.}
\label{3.6}
\end{eqnarray}%
Given that $\mathcal{T}$ is a $\kappa$-averaged and quasi-nonexpansive mapping, we deduce that
\begin{eqnarray}
\left\Vert p_{k+1}-p^{\ast }\right\Vert ^{2} &=&\left\Vert \left( \left(
1-\mathcal{E}_k\right) w_{k}+\mathcal{E}_k\mathcal{T}w_{k}\right) -p^{\ast }\right\Vert ^{2}  \nonumber
\\
&=&\left( 1-\mathcal{E}_k\right) \left\Vert w_{k}-p^{\ast }\right\Vert
^{2}+\mathcal{E}_k\left\Vert \mathcal{T}w_{k}-p^{\ast }\right\Vert ^{2}-\mathcal{E}_k\left(
1-\mathcal{E}_k\right) \left\Vert w_{k}-\mathcal{T}w_{k}\right\Vert ^{2}  \nonumber \\
&\leq &\left( 1-\mathcal{E}_k\right) \left\Vert w_{k}-p^{\ast }\right\Vert
^{2}+\mathcal{E}_k\left( \left\Vert w_{k}-p^{\ast }\right\Vert ^{2}\right.  \nonumber
\\
&&\left. -\left( \frac{1-\kappa }{\kappa }\right) \left\Vert
w_{k}-\mathcal{T}w_{k}\right\Vert ^{2}\right) \text{.} \label{3.7}
\end{eqnarray}%
Using (\ref{3.2}) and (\ref{3.6}) in (\ref{3.7}), we obtain%
\begin{eqnarray*}
	\left\Vert p_{k+1}-p^{\ast }\right\Vert ^{2} &\leq &\left( 1-\mathcal{E}_k\right)
	\left\Vert w_{k}-p^{\ast }\right\Vert ^{2}+\mathcal{E}_k\left( \left\Vert
	w_{k}-p^{\ast }\right\Vert ^{2}-\left( \frac{1-\kappa }{\kappa }\right)
	\left\Vert w_{k}-\mathcal{T}w_{k}\right\Vert ^{2}\right) \\%
	&\leq &\left\Vert w_{k}-p^{\ast }\right\Vert ^{2}-\mathcal{E}_1\left( \frac{1-\kappa
	}{\kappa }\right) \left\Vert w_{k}-\mathcal{T}w_{k}\right\Vert ^{2} \\%
	&\leq &\left( 1+\vartheta \right) \left\Vert p_{k}-p^{\ast }\right\Vert
	^{2}-\left( \vartheta -\delta \right) \left\Vert p_{k-1}-p^{\ast }\right\Vert
	^{2}-\delta \left\Vert p_{k-2}-p^{\ast }\right\Vert ^{2} \\%
	&&+\left( 1+\vartheta \right) \left( \vartheta -\delta \right) \left\Vert
	p_{k}-p_{k-1}\right\Vert ^{2}+\delta \left( 1+\vartheta \right) \left\Vert
	p_{k}-p_{k-2}\right\Vert ^{2} \\%
	&&-\delta \left( \vartheta -\delta \right) \left\Vert
	p_{k-1}-p_{k-2}\right\Vert ^{2}+\delta \left( 1+\vartheta \right) \left\Vert
	p_{k}-p_{k-2}\right\Vert ^{2} \\%
	&&-\mathcal{E}_1\left( \frac{1-\kappa }{\kappa }\right) \left( \left( 1-\left\vert
	\delta \right\vert -\vartheta \right) \left\Vert p_{k+1}-p_{k}\right\Vert
	^{2}-\left( \vartheta ^{2}-\vartheta -\left\vert \delta \right\vert \vartheta \right)
	\left\Vert p_{k}-p_{k-1}\right\Vert ^{2}\right. \\%
	&&\left. +\left( \delta ^{2}-\left\vert \delta \right\vert -\left\vert
	\delta \right\vert \vartheta \right) \left\Vert p_{k-1}-p_{k-2}\right\Vert
	^{2}\right) \\%
	&=&\left( 1+\vartheta \right) \left\Vert p_{k}-p^{\ast }\right\Vert ^{2}-\left(
	\vartheta -\delta \right) \left\Vert p_{k-1}-p^{\ast }\right\Vert ^{2}-\delta
	\left\Vert p_{k-2}-p^{\ast }\right\Vert ^{2} \\%
	&&+\left( \left( 1+\vartheta \right) \left( \vartheta -\delta \right) -\mathcal{E}_1\left(
	\frac{1-\kappa }{\kappa }\right) \left( \vartheta ^{2}-\vartheta -\left\vert \delta
	\right\vert \vartheta \right) \right) \left\Vert p_{k}-p_{k-1}\right\Vert ^{2}
	\\%
	&&-\left( \mathcal{E}_1\left( \frac{1-\kappa }{\kappa }\right) \left( 1-\left\vert
	\delta \right\vert -\vartheta \right) \right) \left\Vert
	p_{k+1}-p_{k}\right\Vert ^{2} \\%
\end{eqnarray*}%
\begin{eqnarray*}
	&&-\left( \delta \left( \vartheta -\delta \right) +\mathcal{E}_1\left( \frac{1-\kappa }{%
		\kappa }\right) \left( \delta ^{2}-\left\vert \delta \right\vert -\left\vert
	\delta \right\vert \vartheta \right) \right) \left\Vert
	p_{k-1}-p_{k-2}\right\Vert ^{2} \\%
	&&+\delta \left( 1+\vartheta \right) \left\Vert p_{k}-p_{k-2}\right\Vert ^{2} \\%
	&\leq &\left( 1+\vartheta \right) \left\Vert p_{k}-p^{\ast }\right\Vert
	^{2}-\left( \vartheta -\delta \right) \left\Vert p_{k-1}-p^{\ast }\right\Vert
	^{2}-\delta \left\Vert p_{k-2}-p^{\ast }\right\Vert ^{2} \\%
	&&+\left( \left( 1+\vartheta \right) \left( \vartheta -\delta \right) -\mathcal{E}_1\left(
	\frac{1-\kappa }{\kappa }\right) \left( \vartheta ^{2}-\vartheta -\left\vert \delta
	\right\vert \vartheta \right) \right) \left\Vert p_{k}-p_{k-1}\right\Vert ^{2}
	\\%
	&&-\left( \mathcal{E}_1\left( \frac{1-\kappa }{\kappa }\right) \left( 1-\left\vert
	\delta \right\vert -\vartheta \right) \right) \left\Vert
	p_{k+1}-p_{k}\right\Vert ^{2} \\%
	&&-\left( \delta \left( \vartheta -\delta \right) +\mathcal{E}_1\left( \frac{1-\kappa }{%
		\kappa }\right) \left( \delta ^{2}-\left\vert \delta \right\vert -\left\vert
	\delta \right\vert \vartheta \right) \right) \left\Vert
	p_{k-1}-p_{k-2}\right\Vert ^{2} \text{.}
\end{eqnarray*}
Therefore, we have%
\begin{eqnarray}
&&\left\Vert p_{k+1}-p^{\ast }\right\Vert ^{2}-\vartheta \left\Vert
p_{k}-p^{\ast }\right\Vert ^{2}-\delta \left\Vert p_{k-1}-p^{\ast
}\right\Vert ^{2}+\left( \mathcal{E}_1\left( \frac{1-\kappa }{\kappa }\right) \left(
1-\left\vert \delta \right\vert -\vartheta \right) \right) \left\Vert
p_{k+1}-p_{k}\right\Vert ^{2}  \nonumber \\
&\leq &\left\Vert p_{k}-p^{\ast }\right\Vert ^{2}-\vartheta \left\Vert
p_{k-1}-p^{\ast }\right\Vert ^{2}-\delta \left\Vert p_{k-2}-p^{\ast
}\right\Vert ^{2}+\left( \mathcal{E}_1\left( \frac{1-\kappa }{\kappa }\right) \left(
1-\left\vert \delta \right\vert -\vartheta \right) \right) \left\Vert
p_{k}-p_{k-1}\right\Vert ^{2}  \nonumber \\
&&+\left( \left( 1+\vartheta \right) \left( \vartheta -\delta \right) -\mathcal{E}_1\left(
\frac{1-\kappa }{\kappa }\right) \left( \vartheta ^{2}-2\vartheta -\left\vert \delta
\right\vert \vartheta -\left\vert \delta \right\vert +1 \right)\right)
\left\Vert p_{k}-p_{k-1}\right\Vert ^{2}  \nonumber \\
&&-\left( \delta \left( \vartheta -\delta \right) +\mathcal{E}_1\left( \frac{1-\kappa }{%
\kappa }\right) \left( \delta ^{2}-\left\vert \delta \right\vert -\left\vert
\delta \right\vert \vartheta \right) \right) \left\Vert
p_{k-1}-p_{k-2}\right\Vert ^{2} \text{.}  \label{3.8}
\end{eqnarray}%
For each $\forall k\geq 1$, define%
\[
\Gamma _{k}:=\left\Vert p_{k}-p^{\ast }\right\Vert ^{2}-\vartheta \left\Vert
p_{k-1}-p^{\ast }\right\Vert ^{2}-\delta \left\Vert p_{k-2}-p^{\ast
}\right\Vert ^{2}+\left( \mathcal{E}_1\left( \frac{1-\kappa }{\kappa }\right) \left(
1-\left\vert \delta \right\vert -\vartheta \right) \right) \left\Vert
p_{k}-p_{k-1}\right\Vert ^{2} \text{.}
\]%
We begin by proving that $\Gamma_k \geq 0$ for all $k \geq 1$. Observe that
\[
\left\Vert p_{k-1}-p^{\ast }\right\Vert ^{2}\leq 2\left\Vert
p_{k}-p_{k-1}\right\Vert ^{2}+2\left\Vert p_{k}-p^{\ast }\right\Vert ^{2}%
\text{.}
\]%
Hence, we get%
\begin{eqnarray}
\Gamma _{k} &=&\left\Vert p_{k}-p^{\ast }\right\Vert ^{2}-\vartheta \left\Vert
p_{k-1}-p^{\ast }\right\Vert ^{2}-\delta \left\Vert p_{k-2}-p^{\ast
}\right\Vert ^{2}+\left( \mathcal{E}_1\left( \frac{1-\kappa }{\kappa }\right) \left(
1-\left\vert \delta \right\vert -\vartheta \right) \right) \left\Vert
p_{k}-p_{k-1}\right\Vert ^{2}  \nonumber \\
&\geq &\left\Vert p_{k}-p^{\ast }\right\Vert ^{2}-2\vartheta \left\Vert
p_{k}-p_{k-1}\right\Vert ^{2}-2\vartheta \left\Vert p_{k}-p^{\ast }\right\Vert
^{2}-\delta \left\Vert p_{k-2}-p^{\ast }\right\Vert ^{2}  \nonumber \\
&&+\left( \mathcal{E}_1\left( \frac{1-\kappa }{\kappa }\right) \left( 1-\left\vert
\delta \right\vert -\vartheta \right) \right) \left\Vert
p_{k}-p_{k-1}\right\Vert ^{2}  \nonumber \\
&=&\left( 1-2\vartheta \right) \left\Vert p_{k}-p^{\ast }\right\Vert
^{2}+\left( \mathcal{E}_1\left( \frac{1-\kappa }{\kappa }\right) \left( 1-\left\vert
\delta \right\vert -\vartheta \right) -2\vartheta \right) \left\Vert
p_{k}-p_{k-1}\right\Vert ^{2}  \nonumber \\
&&-\delta \left\Vert p_{k-2}-p^{\ast }\right\Vert ^{2}\text{.}  \label{3.9}
\end{eqnarray}%
From Assumption B (i)-(iii), we deduce that%
\[
\left\vert \delta \right\vert <\frac{\mathcal{E}_1\left( 1-\kappa \right) \left(
1-\vartheta \right) -2\kappa \vartheta }{\mathcal{E}_1\left( 1-\kappa \right) }\text{.}
\]%
As a result, it follows from (\ref{3.9}) that the inequality $\Gamma_k \geq 0$ holds for every $k \geq 0$. Given that%
\[
-\frac{\mathcal{E}_1\left( 1-\kappa \right) \left( 1-\vartheta \right) -2\kappa \vartheta }{%
\mathcal{E}_1\left( 1-\kappa \right) }<\delta
\]%
and $\vartheta <\frac{1}{3}$ based on Assumption B (i)-(iii). Presently,
we derive from (\ref{3.8}) that%
\begin{eqnarray}
\Gamma _{k+1}-\Gamma _{k} &\leq &\left( \left( 1+\vartheta \right) \left(
\vartheta -\delta \right) -\mathcal{E}_1\left( \frac{1-\kappa }{\kappa }\right) \left(
\vartheta ^{2}-2\vartheta -\left\vert \delta \right\vert \vartheta -\left\vert \delta
\right\vert +1\right) \right) \left\Vert p_{k}-p_{k-1}\right\Vert ^{2}
\nonumber \\
&&-\left( \delta \left( \vartheta -\delta \right) +\mathcal{E}_1\left( \frac{1-\kappa }{%
\kappa }\right) \left( \delta ^{2}-\left\vert \delta \right\vert -\left\vert
\delta \right\vert \vartheta \right) \right) \left\Vert
p_{k-1}-p_{k-2}\right\Vert ^{2}  \nonumber \\
&&-\left( \left( 1+\vartheta \right) \left( \vartheta -\delta \right) -\mathcal{E}_1\left(
\frac{1-\kappa }{\kappa }\right) \left( \vartheta ^{2}-2\vartheta -\left\vert \delta
\right\vert \vartheta -\left\vert \delta \right\vert +1\right) \right)
\nonumber \\
&&\left( \left\Vert p_{k-1}-p_{k-2}\right\Vert ^{2}-\left\Vert
p_{k}-p_{k-1}\right\Vert ^{2}\right)  \nonumber \\
&&+\left( \left( 1+\vartheta \right) \left( \vartheta -\delta \right) -\mathcal{E}_1\left(
\frac{1-\kappa }{\kappa }\right) \left( \vartheta ^{2}-2\vartheta -\left\vert \delta
\right\vert \vartheta -\left\vert \delta \right\vert +1\right) -\delta \left(
\vartheta -\delta \right) \right.  \nonumber \\
&&\left. -\mathcal{E}_1\left( \frac{1-\kappa }{\kappa }\right) \left( \delta
^{2}-\left\vert \delta \right\vert -\left\vert \delta \right\vert \vartheta
\right) \right) \left\Vert p_{k-1}-p_{k-2}\right\Vert ^{2}  \nonumber \\
&\leq &c_{1}\left( \left\Vert p_{k-1}-p_{k-2}\right\Vert ^{2}-\left\Vert
p_{k}-p_{k-1}\right\Vert ^{2}\right) -c_{2}\left\Vert
p_{k-1}-p_{k-2}\right\Vert ^{2} \text{,}  \label{3.10}
\end{eqnarray}%
where%
\begin{eqnarray*}
c_{1} &=&-\left( \left( 1+\vartheta \right) \left( \vartheta -\delta \right)
-\mathcal{E}_1\left( \frac{1-\kappa }{\kappa }\right) \left( \vartheta ^{2}-2\vartheta
-\left\vert \delta \right\vert \vartheta -\left\vert \delta \right\vert
+1\right) \right) \\
c_{2} &=&-\left( \left( 1+\vartheta \right) \left( \vartheta -\delta \right)
-\mathcal{E}_1\left( \frac{1-\kappa }{\kappa }\right) \left( \vartheta ^{2}-2\vartheta
-\left\vert \delta \right\vert \vartheta -\left\vert \delta \right\vert
+1\right) -\delta \left( \vartheta -\delta \right) \right. \\
&&\left. -\mathcal{E}_1\left( \frac{1-\kappa }{\kappa }\right) \left( \delta
^{2}-\left\vert \delta \right\vert -\left\vert \delta \right\vert \vartheta
\right) \right) \text{.}
\end{eqnarray*}%
Considering $\left\vert \delta \right\vert =-\delta $, we obtain that%
\[
-\left( \left( 1+\vartheta \right) \left( \vartheta -\delta \right) -\mathcal{E}_1\left(
\frac{1-\kappa }{\kappa }\right) \left( \vartheta ^{2}-2\vartheta -\left\vert \delta
\right\vert \vartheta -\left\vert \delta \right\vert +1\right) \right) >0
\]%
which is tantamount to
\begin{equation}
\delta >\frac{\vartheta \left( 1+\vartheta \right) -\mathcal{E}_1\left( \frac{1-\kappa }{%
\kappa }\right) \left( 1-\vartheta \right) ^{2}}{\left( 1+\vartheta \right) \left(
1+\mathcal{E}_1\left( \frac{1-\kappa }{\kappa }\right) \right) } \text{.} \label{3.11}
\end{equation}%
So, according to Assumption B (iii), we observe that $c_{1}>0$. Now, we show that $c_{2}>0$. Since,%
\begin{eqnarray}
&&\left( \left( 1+\vartheta \right) \left( \delta -\vartheta \right) +\mathcal{E}_1\left(
\frac{1-\kappa }{\kappa }\right) \left( \vartheta ^{2}-2\vartheta -\left\vert \delta
\right\vert \vartheta -\left\vert \delta \right\vert +1\right) +\delta \left(
\vartheta -\delta \right) \right.  \nonumber \\
&&\left. +\mathcal{E}_1\left( \frac{1-\kappa }{\kappa }\right) \left( \delta
^{2}-\left\vert \delta \right\vert -\left\vert \delta \right\vert \vartheta
\right) \right)  \nonumber \\
&>&0  \text{,}\label{3.12}
\end{eqnarray}%
we get%
\begin{eqnarray}
	\delta \left( \vartheta -\delta \right) +\delta \left( \vartheta +1\right) \left[
	1+\mathcal{E}_1\left( \frac{1-\kappa }{\kappa }\right) \right] +\mathcal{E}_1\left( \frac{
		1-\kappa }{\kappa }\right) \left( \delta ^{2}+\delta \left( \vartheta +1\right)
	\right) \nonumber \\
	< \vartheta \left( \vartheta +1\right) -\mathcal{E}_1\left( \frac{1-\kappa }{\kappa }
	\right) \left( \vartheta -1\right) ^{2} \label{3.13}
\end{eqnarray}
So, it is clear from Assumption B (iii) that $c_{2}>0$. On the other hand, we derive from (\ref{3.10})that%
\begin{equation}
\Gamma _{k+1}+c_{1}\left\Vert p_{k}-p_{k-1}\right\Vert ^{2}\leq \Gamma
_{k}+c_{1}\left\Vert p_{k-1}-p_{k-2}\right\Vert ^{2}-c_{2}\left\Vert
p_{k-1}-p_{k-2}\right\Vert ^{2}\text{.}  \label{3.14}
\end{equation}%
By defining $\overline{\Gamma }_{k}:=\Gamma _{k}+c_{1}\left\Vert
p_{k-1}-p_{k-2}\right\Vert ^{2}$, we deduce from (\ref{3.14}) that $\overline{%
\Gamma }_{k+1}\leq \overline{\Gamma }_{k}$, which indicates that the
sequence $\left\{ \overline{\Gamma }_{k}\right\} $ is decreasing and
therefore, the limit $\lim_{k\rightarrow \infty }\overline{\Gamma }_{k}$
exists. Consequently, based on (\ref{3.14}), we conclude that $%
\lim_{k\rightarrow \infty }\Gamma _{k}+c_{1}\left\Vert
p_{k-1}-p_{k-2}\right\Vert ^{2}=0$. Thus, we have%
\begin{equation}
\lim_{k\rightarrow \infty }\left\Vert p_{k-1}-p_{k-2}\right\Vert =0\text{.}
\label{3.15}
\end{equation}%
By employing (\ref{3.15}) and considering the existence of the limit of $%
\left\{ \overline{\Gamma }_{k}\right\} $, it follows that%
\begin{eqnarray}
	\lim_{k\rightarrow \infty }\Gamma _{k} &=& \lim_{k\rightarrow \infty }
	\left( \left\Vert p_{k}-p^{\ast }\right\Vert ^{2}-\vartheta \left\Vert
	p_{k-1}-p^{\ast }\right\Vert ^{2}-\delta \left\Vert p_{k-2}-p^{\ast
	}\right\Vert ^{2} \right.  \nonumber \\
	&& \left. + \mathcal{E}_1\left( \frac{1-\kappa }{\kappa }\right) \left( 1-\left\vert
	\delta \right\vert -\vartheta \right) \left\Vert p_{k}-p_{k-1}\right\Vert ^{2} \right)
	\label{3.16}
\end{eqnarray}
exists. Also, we have%
\begin{eqnarray*}
\left\Vert p_{k+1}-w_{k}\right\Vert &=&\left\Vert p_{k+1}-p_{k}-\vartheta
\left( p_{k}-p_{k-1}\right) -\delta \left( p_{k-1}-p_{k-2}\right) \right\Vert
\\
&\leq &\left\Vert p_{k+1}-p_{k}\right\Vert -\vartheta \left\Vert
p_{k}-p_{k-1}\right\Vert -\left\vert \delta \right\vert \left\Vert
p_{k-1}-p_{k-2}\right\Vert \rightarrow 0\text{, }k\rightarrow \infty.
\end{eqnarray*}%
So, we obtain $\lim_{k\rightarrow \infty }\left\Vert w_{k}-\mathcal{T}w_{k}\right\Vert
=0$. It is also worth noting that
\begin{eqnarray*}
\left\Vert w_{k}-p_{k}\right\Vert &=&\left\Vert p_{k}+\vartheta \left(
p_{k}-p_{k-1}\right) +\delta \left( p_{k-1}-p_{k-2}\right) -p_{k}\right\Vert
\\
&=&\left\Vert \vartheta \left( p_{k}-p_{k-1}\right) +\delta \left(
p_{k-1}-p_{k-2}\right) \right\Vert \\
&\leq &\vartheta \left\Vert p_{k}-p_{k-1}\right\Vert +\left\vert \delta
\right\vert \left\Vert p_{k-1}-p_{k-2}\right\Vert \rightarrow 0\text{, }%
k\rightarrow \infty.
\end{eqnarray*}%
Since the limit $\lim_{k\rightarrow \infty }\Gamma _{k}$ exists and $%
\lim_{k\rightarrow \infty }\left\Vert p_{k}-p_{k-1}\right\Vert =0$, it
follows from (\ref{3.9}) that $\left\{ p_{k}\right\} $ is bounded.
\end{proof}
Now, we prove the weak convergence of the sequence generated by Algorithm 1*.
\begin{theorem}
\textbf{\label{The3.1} } Assuming that Assumption A holds, the sequence $\{p_k\}$ produced by Algorithm 1* converges weakly to an element of $(A+B)^{-1}(0)$.
\end{theorem}

\begin{proof}Applying (\ref{3.15}) within (\ref{3.16}), we establish the
existence of the following limit as $k$ tends to infinity:%
\[
\lim_{k\rightarrow \infty }\left[ \left\Vert p_{k}-p^{\ast }\right\Vert
^{2}-\vartheta \left\Vert p_{k-1}-p^{\ast }\right\Vert ^{2}-\delta \left\Vert
p_{k-2}-p^{\ast }\right\Vert ^{2}\right] \text{.}
\]%
According to Lemma \ref{Lem3.1}, we know that $\left\{ p_{k}\right\} $ is bounded.
Initially, we establish that any weak cluster point of the sequence $\{p_k\}$ belongs to $F(\mathcal{T})$, where the operator $\mathcal{T}$ is defined in (\ref{3.1}).
Assume $\left\{ p_{k_{n}}\right\} \subset \left\{ p_{k}\right\} $ and $%
p_{k_{n}}\rightharpoonup v^{\ast }\in H$. Since $\left\Vert
w_{k}-p_{k}\right\Vert \rightarrow 0$, $k\rightarrow \infty $, we have $%
w_{k_{n}}\rightharpoonup v^{\ast }\in H$. Eventually $\left\Vert
w_{k}-\mathcal{T}w_{k}\right\Vert \rightarrow 0$, $k\rightarrow \infty $, and by the
demiclosedness of $\mathcal{T}$, we deduce that $v^{\ast }\in F(\mathcal{T})=\left( A+B\right)
^{-1}\left( 0\right) $. Now, we show that $p_{k}\rightharpoonup p^{\ast
}\in F(\mathcal{T})$. Let assume the existence of $\left\{ p_{k_{n}}\right\} \subset
\left\{ p_{k}\right\} $ and $\left\{ p_{k_{j}}\right\} \subset \left\{
p_{k}\right\} $ such that $p_{k_{n}}\rightharpoonup v^{\ast }$, $%
n\rightarrow \infty $ and $p_{k_{j}}\rightharpoonup p^{\ast }$, $%
j\rightarrow \infty $. We show that $v^{\ast }=p^{\ast }$%
\begin{eqnarray}
2\left\langle p_{k},p^{\ast }-v^{\ast }\right\rangle &=&\left\Vert
p_{k}-v^{\ast }\right\Vert ^{2}-\left\Vert p_{k}-p^{\ast }\right\Vert ^{2}
\nonumber \\
&&-\left\Vert v^{\ast }\right\Vert ^{2}+\left\Vert p^{\ast }\right\Vert ^{2}
\label{3.17} \\
2\left\langle -\vartheta p_{k-1},p^{\ast }-v^{\ast }\right\rangle &=&-\vartheta
\left\Vert p_{k-1}-v^{\ast }\right\Vert ^{2}+\vartheta \left\Vert
p_{k-1}-p^{\ast }\right\Vert ^{2}  \nonumber \\
&&+\vartheta \left\Vert v^{\ast }\right\Vert ^{2}-\vartheta \left\Vert p^{\ast
}\right\Vert ^{2}  \label{3.18}
\end{eqnarray}%
and%
\begin{eqnarray}
2\left\langle -\delta p_{k-2},p^{\ast }-v^{\ast }\right\rangle &=&-\delta
\left\Vert p_{k-2}-v^{\ast }\right\Vert ^{2}+\delta \left\Vert
p_{k-2}-p^{\ast }\right\Vert ^{2}  \nonumber \\
&&+\delta \left\Vert v^{\ast }\right\Vert ^{2}-\delta \left\Vert p^{\ast
}\right\Vert ^{2}\text{.}  \label{3.19}
\end{eqnarray}%
Addition of (\ref{3.17}), (\ref{3.18}), and (\ref{3.19}) gives%
\begin{eqnarray*}
2\left\langle p_{k}-\vartheta p_{k-1}-\delta p_{k-2},p^{\ast }-v^{\ast
}\right\rangle &=&\left( \left\Vert p_{k}-v^{\ast }\right\Vert ^{2}-\vartheta
\left\Vert p_{k-1}-v^{\ast }\right\Vert ^{2}-\delta \left\Vert
p_{k-2}-v^{\ast }\right\Vert ^{2}\right) \\
&&-\left( \left\Vert p_{k}-p^{\ast }\right\Vert ^{2}-\vartheta \left\Vert
p_{k-1}-p^{\ast }\right\Vert ^{2}-\delta \left\Vert p_{k-2}-p^{\ast
}\right\Vert ^{2}\right) \\
&&+\left( 1-\vartheta -\delta \right) \left( \left\Vert p^{\ast }\right\Vert
^{2}-\left\Vert v^{\ast }\right\Vert ^{2}\right) \text{.}
\end{eqnarray*}%
According to (\ref{3.16}), the following limits are exist:%
\[
\lim_{k\rightarrow \infty }\left\Vert p_{k}-p^{\ast }\right\Vert ^{2}-\vartheta
\left\Vert p_{k-1}-p^{\ast }\right\Vert ^{2}-\delta \left\Vert
p_{k-2}-p^{\ast }\right\Vert ^{2}
\]%
exists and%
\[
\lim_{k\rightarrow \infty }\left\Vert p_{k}-v^{\ast }\right\Vert ^{2}-\vartheta
\left\Vert p_{k-1}-v^{\ast }\right\Vert ^{2}-\delta \left\Vert
p_{k-2}-v^{\ast }\right\Vert ^{2}\text{.}
\]%
This implies that%
\[
\lim_{k\rightarrow \infty }\left\langle p_{k}-\vartheta p_{k-1}-\delta
p_{k-2},p^{\ast }-v^{\ast }\right\rangle
\]%
exists. So, we have%
\begin{eqnarray*}
\left\langle v^{\ast }-\vartheta v^{\ast }-\delta v^{\ast },p^{\ast }-v^{\ast
}\right\rangle &=&\lim_{n\rightarrow \infty }\left\langle p_{k_{n}}-\vartheta
p_{k_{n}-1}-\delta p_{k_{n}-2},p^{\ast }-v^{\ast }\right\rangle \\
&=&\lim_{k\rightarrow \infty }\left\langle p_{k}-\vartheta p_{k-1}-\delta
p_{k-2},p^{\ast }-v^{\ast }\right\rangle \\
&=&\lim_{j\rightarrow \infty }\left\langle p_{k_{j}}-\vartheta
p_{k_{j}-1}-\delta p_{k_{j}-2},p^{\ast }-v^{\ast }\right\rangle \\
&=&\left\langle p^{\ast }-\vartheta p^{\ast }-\delta p^{\ast },p^{\ast
}-v^{\ast }\right\rangle \text{,}
\end{eqnarray*}%
which yields%
\[
\left( 1-\vartheta -\delta \right) \left\Vert p^{\ast }-v^{\ast }\right\Vert
^{2}=0\text{.}
\]%
Since $\delta \leq 0<1-\vartheta $, we obtain that $p^{\ast }=v^{\ast }$.
Hence, $\left\{ p_{k}\right\} $ converges weakly to a point in $F(\mathcal{T})=\left(
A+B\right) ^{-1}\left( 0\right) $.
\end{proof}

\section{Numerical Simulations}

We designed a new algorithm called as Algorithm 1* based on extreme learning machine network
(ELM) with Python programming language. In the implementation phase of the
ELM, "Linear Activation Function" and "Sigmoid Activation Function" are
chosen as activation functions. MSE (Mean Squared Error), RMSE (Root Mean
Squared Error), R$^{2}$ score and MAE (Mean Absolute Error) are used as
performance criteria to train and test phase. We also analyzed their
performance in terms of time value. All codes were run on a computer with
the following specifications: Intel(R) Core(TM) i5-7200U CPU @ 2.50GHz 2.71
GHz, 8.0 GB installed RAM. Eventually performance of our algorithm is
compared with different algorithms (Fista in \cite{26}, Vifba in
\cite{27}, Algorithm 1 in \cite{3} and Algorithm 5 in \cite{29}).
\\
We improved the results by increasing the number of iterations and the
number of Hidden Nodes due to avoid of overfitting problem. The purpose of
the number of Hidden Nodes is to enable the neural network to learn more
complex relationships. The number of Hidden Nodes acts as a filter that
processes the inputs and then sends them to the output layer. Under all
these common conditions, we examined our algorithm compared with other
algorithms for the Regression and Data Classification problems. First,
Regression analysis was performed using the linear activation function for
the sine curve. As a result of this comparison, Figure \ref{Fig1} obtained from 1000 iterations for randomly selected 10 sine values is presented below.%

\begin{figure}[H]
\par
\begin{center}
\includegraphics[width=0.8\textwidth]{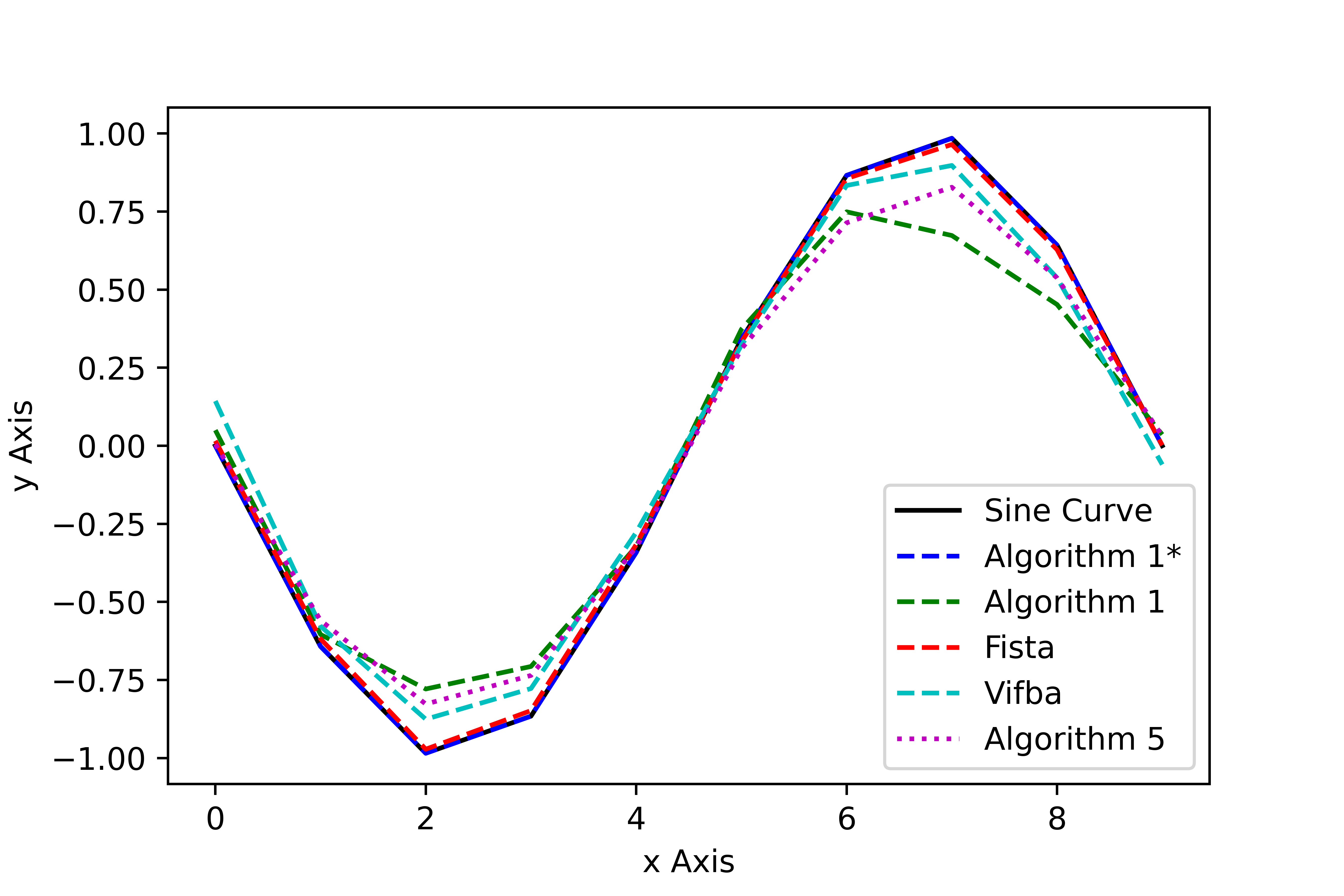} %
\end{center}
\caption{Regression analysis using linear activation function for random 10 sine values}
\label{Fig1}
\end{figure}
Figure \ref{Fig1} shows that the designed model exhibited the best
convergent behavior to the sinusoidal wave. We then performed a statistical performance analysis. The
results are given in Table \ref{Tab1}, which compares success ranking and time according to the error test results with other algorithms. As can be seen from Table \ref{Tab1}, Algorithm 1* works faster and more successfully than all other algorithms.
\begin{table}[h]
\centering
\begin{tabular}[b]{|c|c|c|c|c|c|}
\hline
\textbf{Algorithms} & \textbf{CPU Time} & \textbf{MSE} & \textbf{RMSE} &
\textbf{R}$^{2}$\textbf{\ Score} & \textbf{MAE} \\ \hline
\emph{Algorithm 1*} & 0,0061380 & 5,36789E-10 & 2,31687E-05 & 0,99999999
& 2,07639E-05 \\ \hline
Fista & 0,0175220 & 0,000277777 & 0,016666639 & 0,99938271 & 0,015177549
\\ \hline
Vifba & 0,0146362 & 0,007225812 & 0,085004778 & 0,98394263 & 0,077460916
\\ \hline
Algorithm 5 & 0,0314197 & 0,010976636 & 0,104769445 & 0,97560747 &
0,086961912 \\ \hline
Algorithm 1 & 0,0157370 & 0,022115578 & 0,148713072 & 0,95085427 &
0,115831474 \\ \hline
\end{tabular}
\caption{Statistical performance analysis for linear activation function and random 10 sine values}
\label{Tab1}
\end{table}

After these results, we realized simulations using linear activation
function for 1000 iterations and randomly selected 100 sine values. Observed
results are given in Figure \ref{Fig2}. The statistical results of the
simulations are also given in Table \ref{Tab2}.
\begin{figure}[H]
\par
\begin{center}
\includegraphics[width=0.8\textwidth]{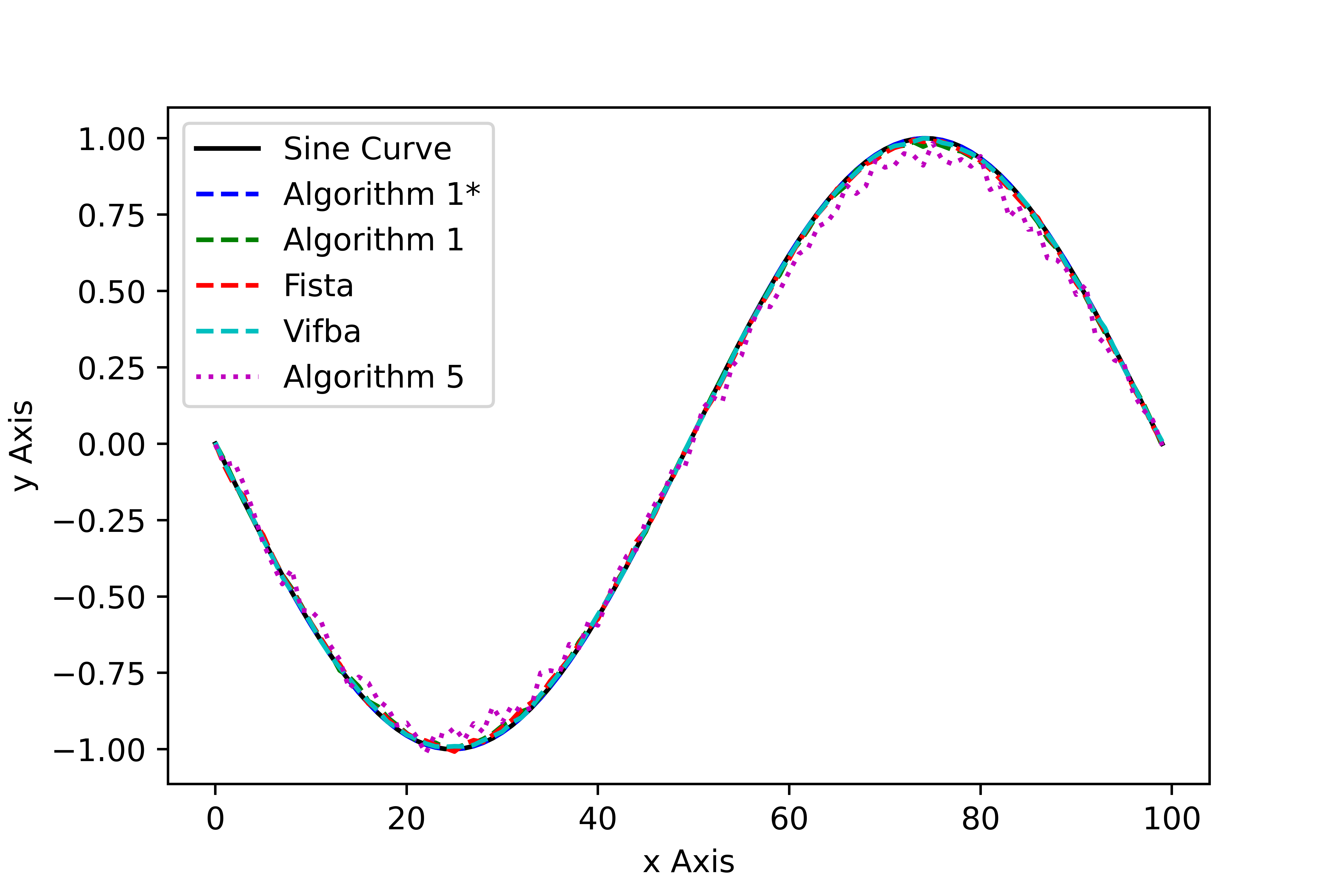} %
\end{center}
\caption{Regression analysis using linear activation function for random 100 sine values}
\label{Fig2}
\end{figure}
\begin{table}[h]
\centering
\begin{tabular}[b]{|c|c|c|c|c|c|}
\hline
\textbf{Algorithms} & \textbf{CPU Time} & \textbf{MSE} & \textbf{RMSE} &
\textbf{R}$^{2}$\textbf{\ Score} & \textbf{MAE} \\ \hline
\emph{Algorithm 1*} & 0,0114526 & 2,36002E-10 & 1,53624E-05 & 0,99999999
& 1,30988E-05 \\ \hline
Vifba & 0,0312581 & 2,50000E-05 & 0,005000000 & 0,99994949 & 0,004165242
\\ \hline
Fista & 0,0138959 & 9,83294E-05 & 0,009916117 & 0,99980135 & 0,008377328
\\ \hline
Algorithm 5 & 0,0157158 & 9,91574E-05 & 0,009957779 & 0,99979968 &
0,007872083 \\ \hline
Algorithm 1 & 0,0785269 & 0,002296867 & 0,047925644 & 0,99535986 &
0,040710174 \\ \hline
\end{tabular}
\caption{Statistical performance analysis for linear activation function and random 100 sine values}
\label{Tab2}
\end{table}

As can be seen in Figure \ref{Fig2} and Table \ref{Tab2}, Algorithm 1*
has higher regression performance and works faster than other algorithms.
After these successful results, we performed Regression analysis using the
sigmoid activation function for the sine curve. As a result of this
comparison, Figure \ref{Fig3}, obtained from 5000 iterations for 10
random sine values, is presented below.%
\begin{figure}[H]
\par
\begin{center}
\includegraphics[width=0.8\textwidth]{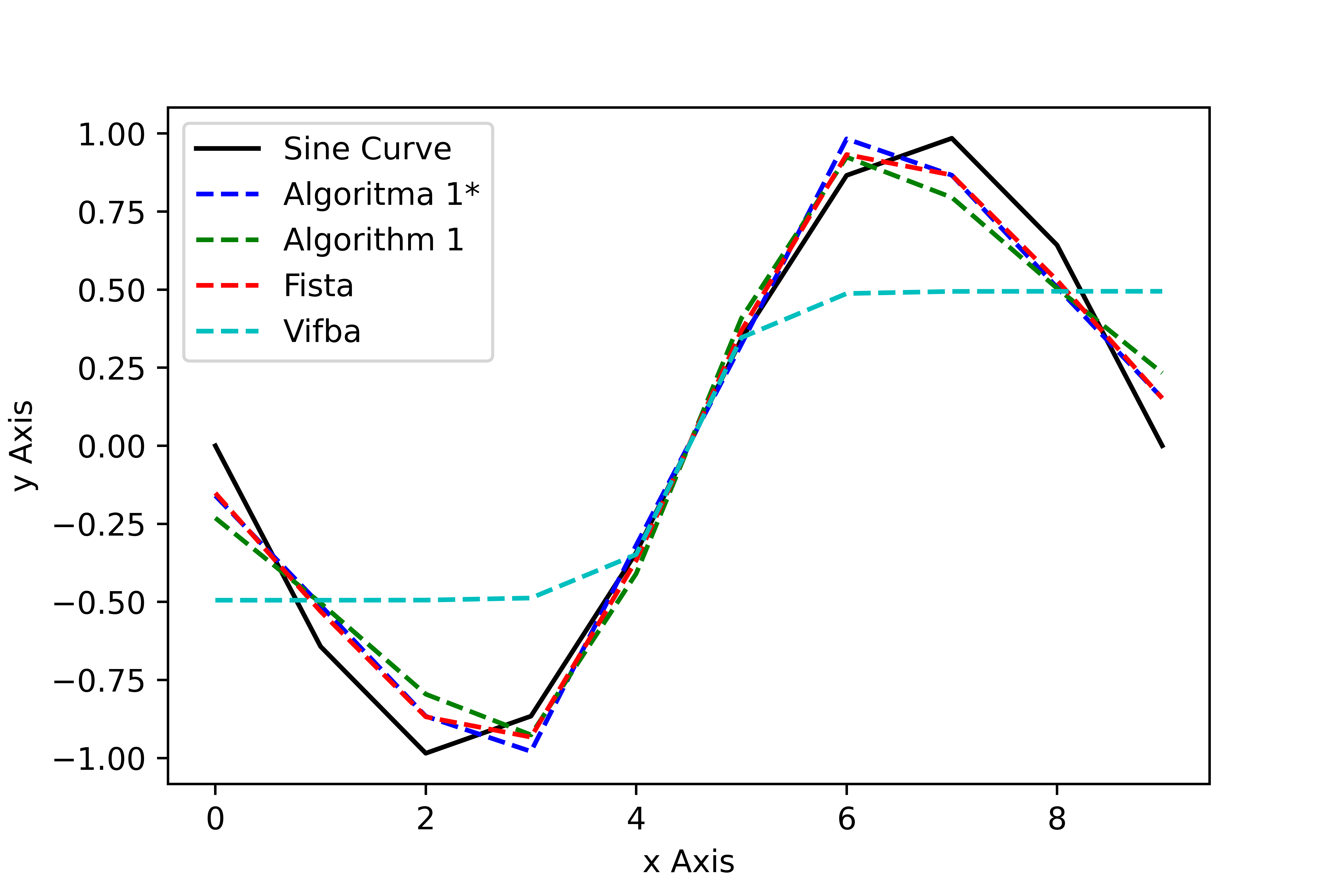} %
\end{center}
\caption{Regression analysis using sigmoid activation function for random 10 sine values}
\label{Fig3}
\end{figure}
In Table \ref{Tab3} below, success ranking and time comparison are made according to the error test results with other algorithms. As can be seen from Table \ref{Tab3}, Algorithm 1* works more successfully than all other algorithms except Fista \cite{26}. According to the time values, Algorithm 1* has the latest reaction time, even though the differences are negligible.%

\begin{table}[h]
\centering
\begin{tabular}[b]{|c|c|c|c|c|c|}
\hline
\textbf{Algorithms} & \textbf{CPU Time} & \textbf{MSE} & \textbf{RMSE} &
\textbf{R}$^{2}$\textbf{\ Score} & \textbf{MAE} \\ \hline
Fista & 0,04846715 & 0,01064577 & 0,10317836 & 0,97634272 & 0,09337873
\\ \hline
\emph{Algorithm 1*} & 0,06319165 & 0,01512668 & 0,12299057 & 0,96638515
& 0,11060022 \\ \hline
Algorithm 1 & 0,06293702 & 0,02238588 & 0,14961913 & 0,95025358 &
0,13500654 \\ \hline
Vifba & 0,06139302 & 0,13004268 & 0,36061432 & 0,71101624 & 0,30326138
\\ \hline
\end{tabular}
\caption{Statistical performance analysis for sigmoid activation function and random 10 sine values.}
\label{Tab3}
\end{table}

Finally, after the results we obtained, we performed Regression analysis
using the sigmoid activation function as a result of 5000 iterations for 100
random sine values. Again, we have compared our results with other algorithms
and presented their figures and tables below.%
\begin{figure}[H]
\par
\begin{center}
\includegraphics[width=0.8\textwidth]{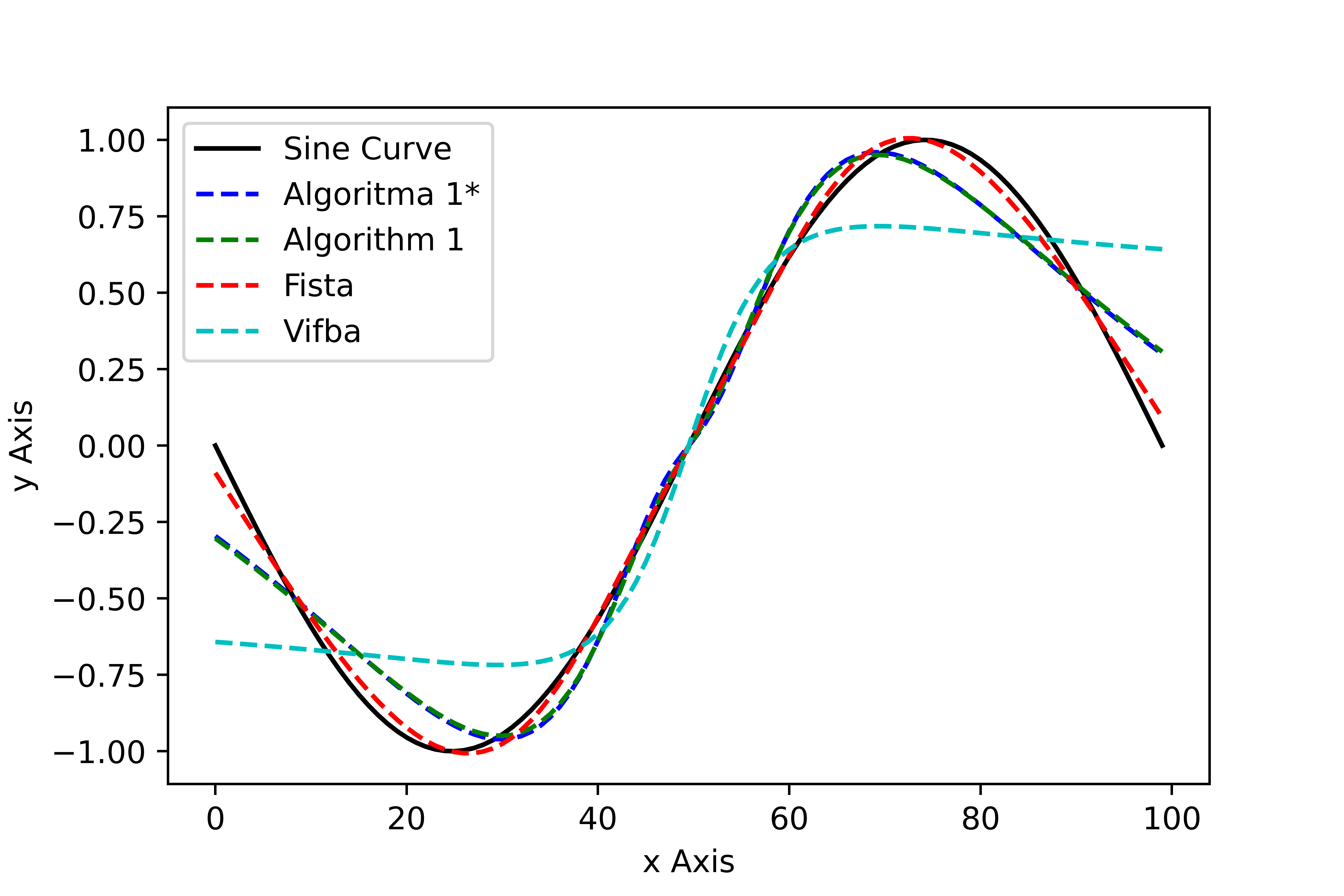} %
\end{center}
\caption{Regression analysis using sigmoid activation function for random 100 sine values.}
\label{Fig4}
\end{figure}
\begin{table}[h]
\centering
\begin{tabular}[b]{|c|c|c|c|c|c|}
\hline
\textbf{Algorithms} & \textbf{CPU Time} & \textbf{MSE} & \textbf{RMSE} &
\textbf{R}$^{2}$\textbf{\ Score} & \textbf{MAE} \\ \hline
Fista & 0,3937718 & 0,001036359 & 0,032192538 & 0,9979063 & 0,026613099
\\ \hline
\emph{Algorithm 1*} & 0,4539487 & 0,012225424 & 0,110568638 & 0,9753021
& 0,089891046 \\ \hline
Algorithm 1 & 0,3770158 & 0,012830698 & 0,113272669 & 0,9740793 &
0,091232123 \\ \hline
Vifba & 0,3823239 & 0,056802116 & 0,238331945 & 0,8852482 & 0,189809277
\\ \hline
\end{tabular}
\caption{Statistical performance analysis for sigmoid activation function and random 100 sine values.}
\label{Tab4}
\end{table}

We now put the Algorithm 1* to the test to solve the data classification problem. We used Iris data set defined by the famous statistician and biologist Sir Ronald Aylmer Fisher in 1936 \cite{30}. This dataset contains features used to classify the species of Iris flowers and is a sample dataset often used for training and evaluating machine learning algorithms. The 150-sample Iris dataset was divided into \%80 training and \%20 testing phases. Simulations were performed with designed algorithms using the sigmoid activation function, and the results obtained are given in Figure \ref{Fig5} and Table \ref{Tab5}. We used confusion matrix method to see the success of Algorithm 1*.
\begin{figure}[H]
\par
\begin{center}
\includegraphics[width=0.8\textwidth]{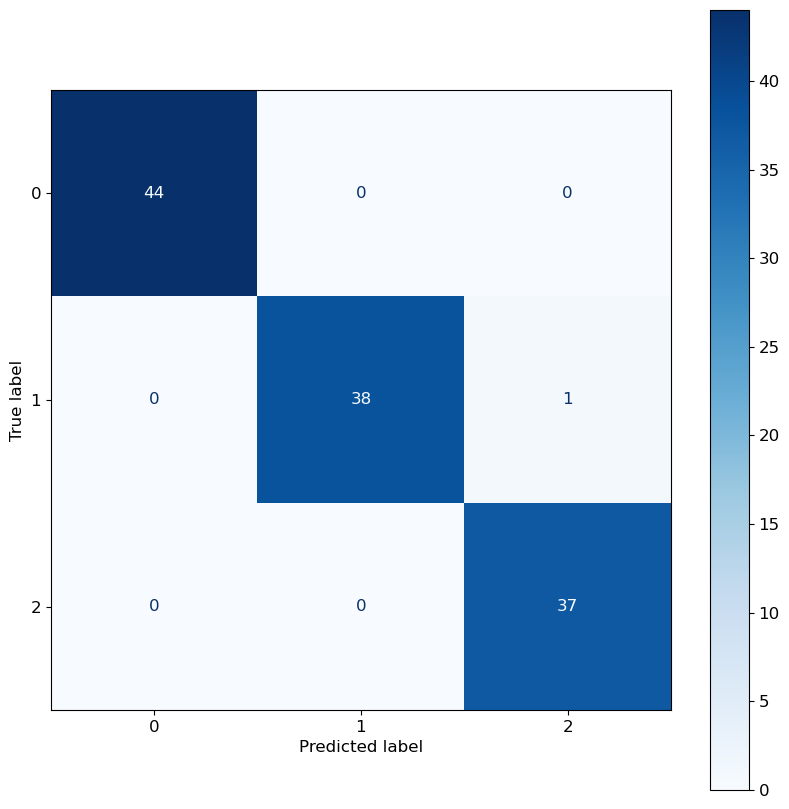} %
\end{center}
\caption{Confusion matrix for Algorithm 1*}
\label{Fig5}
\end{figure}
\begin{table}[h]
	\centering
	\begin{minipage}{\linewidth}
		\centering
		\begin{tabular}{|c|c|c|c|c|}
			\hline
			& Precision & Recall & F1-Score & Support \\ \hline
			0 (Setosa) & 1.00 & 1.00 & 1.00 & 44 \\ \hline
			1 (Versicolor) & 1.00 & 0.97 & 0.99 & 39 \\ \hline
			2 (Virginica) & 0.97 & 1.00 & 0.99 & 37 \\ \hline
		\end{tabular}
	\end{minipage}
	
	\vspace{0.5cm} 
	
	\begin{minipage}{\linewidth}
		\centering
		\begin{tabular}{|c|c|}
			\hline
			Accuracy & 0.99 \\ \hline
			Total Support & 120 \\ \hline
		\end{tabular}
	\end{minipage}
	
	\caption{Performance of Algorithm 1* for data classification problem}
	\label{Tab5}
\end{table}

As can be seen from Figure \ref{Fig5} and Table \ref{Tab5}, the Algorithm 1* exhibited succesfull
performance; \%100, \%97,43 and \%100 for class 0,1,2 respectively. It was
also seen that our algorithm concluded higher rate success for predicting of
truly positive class and total positive classes (TPR).

\section{Conclusion and Discussion}

We propose a forward-backward splitting algorithm to find the zero point of
the sum of maximal monotone and co-coercive operators in real Hilbert
spaces. Unlike most of the algorithms in the literature, the new algorithm has been created that includes double inertial terms. We proved the weak
convergence of the sequence generated by the algorithm and gave applications to regression and data classification problems. When we compare the algorithm we have presented
with other algorithms that give successful results in the literature, we
have obtained more successful results than the others. In our future studies, necessary arrangements will be made to optimize the convergence rate of the algorithm and to provide more successful results. \\

\textbf{Data availability} The author confirms that no real data have been
used in this manuscript. All data are synthetically generated.

\textbf{Competing Interests} The authors declare that they have no conflict
of interest.


\begin{thebibliography}{99}
\bibitem{40} Alt\i parmak, E., \& Karahan, I. (2021). Image Restoration
Using An \.{I}nertial Viscosity Fixed Point Algorithm. Arxiv Preprint
Arxiv:2108.05146.

\bibitem{17} Altiparmak, E., \& Karahan, I. (2022). A Mod\i f\i ed Precond\i
t\i on\i ng Algor\i thm For Solv\i ng Monotone Inclus\i on Problem And
Appl\i cat\i on To Image Restorat\i on Problem. Politeh. Buch. Ser. A, 84,
81-92.

\bibitem{18} Altiparmak, E., \& Karahan, I. (2022). A new preconditioning
algorithm for finding a zero of the sum of two monotone operators and its
application to image restoration problems. International Journal of Computer
Mathematics, 99(12), 2482-2498.

\bibitem{20} Alvarez, F., \& Attouch, H. (2001). An inertial proximal method
for maximal monotone operators via discretization of a nonlinear oscillator
with damping. Set-Valued Analysis, 9, 3-11.

\bibitem{23} Attouch, H., Peypouquet, J., \& Redont, P. (2014). A dynamical
approach to an inertial forward-backward algorithm for convex minimization.
SIAM Journal on Optimization, 24(1), 232-256.

\bibitem{1} Bauschke, H. H., \& Combettes, P. L. Convex Analysis and
Monotone Operator Theory in Hilbert Spaces, 2011. CMS books in mathematics).
DOI, 10, 978-1.

\bibitem{26} Beck, A., \& Teboulle, M. (2009). A fast iterative
shrinkage-thresholding algorithm for linear inverse problems. SIAM journal
on imaging sciences, 2(1), 183-202.

\bibitem{2} Bo\c{t}, R. I., Csetnek, E. R., \& Meier, D. (2019). Inducing
strong convergence into the asymptotic behaviour of proximal splitting
algorithms in Hilbert spaces. Optimization Methods and Software, 34(3),
489-514.

\bibitem{10} Combettes, P. L., \& Wajs, V. R. (2005). Signal recovery by
proximal forward-backward splitting. Multiscale modeling \& simulation,
4(4), 1168-1200.

\bibitem{42} Dixit, A., Sahu, D. R., Singh, A. K., \& Som, T. (2020).
Application of a new accelerated algorithm to regression problems. Soft
Computing, 24, 1539-1552.

\bibitem{4} Dixit, A., Sahu, D. R., Gautam, P., Som, T., \& Yao, J. C.
(2021). An accelerated forward-backward splitting algorithm for solving
inclusion problems with applications to regression and link prediction
problems. choice, 4(1), 2.

\bibitem{32} Dong, Q. L., Cho, Y. J., \& Rassias, T. M. (2018). General
inertial Mann algorithms and their convergence analysis for nonexpansive
mappings. Applications of Nonlinear Analysis, 175-191.

\bibitem{34} Douglas, J., \& Rachford, H. H. (1956). On the numerical
solution of heat conduction problems in two and three space variables.
Transactions of the American mathematical Society, 82(2), 421-439.

\bibitem{30} Fisher, R. A. (1936). The use of multiple measurements in
taxonomic problems. Annals of eugenics, 7(2), 179-188.

\bibitem{43} Huang, G. B., Wang, D. H., \& Lan, Y. (2011). Extreme learning
machines: a survey. International journal of machine learning and
cybernetics, 2, 107-122.

\bibitem{29} Inthakon, W., Suantai, S., Sarnmeta, P., \& Chumpungam, D.
(2020). A new machine learning algorithm based on optimization method for
regression and classification problems. Mathematics, 8(6), 1007.

\bibitem{33} Iyiola, O. S., \& Shehu, Y. (2022). Convergence results of
two-step inertial proximal point algorithm. Applied Numerical Mathematics,
182, 57-75.

\bibitem{3} Jolaoso, L. O., Shehu, Y., YAO, J. C., \& Xu, R. (2023). Double
Inert\i al Parameters Forward-Backward Spl\i tt\i ng Method: Appl\i cat\i
ons To Compressed Sens\i ng, Image Process\i ng, And Scad Penalty Problems.
Journal of Nonlinear \& Variational Analysis, 7(4).

\bibitem{5} Kitkuan, D., Kumam, P., \& Mart\'{\i}nez-Moreno, J. (2020).
Generalized Halpern-type forward--backward splitting methods for convex
minimization problems with application to image restoration problems.
Optimization, 69(7-8), 1557-1581.

\bibitem{35} Lions, P. L., \& Mercier, B. (1979). Splitting algorithms for
the sum of two nonlinear operators. SIAM Journal on Numerical Analysis,
16(6), 964-979.

\bibitem{6} L\'{o}pez, G., Mart\'{\i}n-M\'{a}rquez, V., Wang, F., \& Xu, H.
K. (2012, January). Forward-backward splitting methods for accretive
operators in Banach spaces. In Abstract and Applied Analysis (Vol. 2012).
Hindawi.

\bibitem{7} Lorenz, D. A., \& Pock, T. (2015). An inertial forward-backward
algorithm for monotone inclusions. Journal of Mathematical Imaging and
Vision, 51, 311-325.

\bibitem{16} Moudafi, A. (2000). Viscosity approximation methods for
fixed-points problems. Journal of mathematical analysis and applications,
241(1), 46-55.

\bibitem{15} Moudafi, A., \& Oliny, M. (2003). Convergence of a splitting
inertial proximal method for monotone operators. Journal of Computational
and Applied Mathematics, 155(2), 447-454.

\bibitem{22} Nesterov, Y. E. E. (1983). A method of solving a convex
programming problem with convergence rate $O\left( \frac{1}{k^{2}}\right) $.
In Doklady Akademii Nauk (Vol. 269, No. 3, pp. 543-547). Russian Academy of
Sciences.

\bibitem{24} Peeyada, P., Cholamjiak, W., \& Yambangwai, D. (2022). A hybrid
inertial parallel subgradient extragradient-line algorithm for variational
inequalities with an application to image recovery, J. Nonlinear Funct.
Anal., 2022 (2022), 9.

\bibitem{21} Polyak, B. T. (1964). Some methods of speeding up the
convergence of iteration methods. Ussr computational mathematics and
mathematical physics, 4(5), 1-17.

\bibitem{39} Polyak, B. T. (1987). Introduction to optimization.

\bibitem{38} Poon, C., \& Liang, J. (2019). Trajectory of alternating
direction method of multipliers and adaptive acceleration. Advances in
Neural Information Processing Systems, 32.

\bibitem{37} Poon, C., \& Liang, J. (2020). Geometry of first-order methods
and adaptive acceleration. arXiv preprint arXiv:2003.03910.

\bibitem{19} Rehman, H. U., \"{O}zdem\i r, M., Karahan, \.{I}., \& Wa\i
rojjana, N. (2022). The Tseng's Extragrad\i ent Method For Sem\i str\i ctly
Quas\i monotone Var\i at\i onal Inequal\i t\i es. Journal Of Applied \&
Numerical Optimization, 4(2).

\bibitem{11} Sitthithakerngkiet, K., Deepho, J., \& Kumam, P. (2015). A
hybrid viscosity algorithm via modify the hybrid steepest descent method for
solving the split variational inclusion in image reconstruction and fixed
point problems. Applied Mathematics and Computation, 250, 986-1001.

\bibitem{12} Sra, S., Nowozin, S., \& Wright, S. J. (Eds.). (2012).
Optimization for machine learning. Mit Press.

\bibitem{36} Suantai, S., Inkrong, P., \& Cholamjiak, P. (2023).
Forward--backward--forward algorithms involving two inertial terms for
monotone inclusions. Computational and Applied Mathematics, 42(6), 255.

\bibitem{8} Sunthrayuth, P., \& Cholamjiak, P. (2018). Iterative methods for
solving quasi-variational inclusion and fixed point problem in q-uniformly
smooth Banach spaces. Numerical Algorithms, 78, 1019-1044.

\bibitem{25} Taiwo, A., \& Mewomo, O. T. (2022). Inertial viscosity with
alternative regularization for certain optimization and fixed point
problems. Journal of Applied \& Numerical Optimization, 4(3).

\bibitem{13} Tibshirani, R. (1996). Regression shrinkage and selection via
the lasso. Journal of the Royal Statistical Society Series B: Statistical
Methodology, 58(1), 267-288.

\bibitem{31} Tseng, P. (2000). A modified forward-backward splitting method
for maximal monotone mappings. SIAM Journal on Control and Optimization,
38(2), 431-446.

\bibitem{27} Verma, M., Sahu, D. R., \& Shukla, K. K. (2018). VAGA: a novel
viscosity-based accelerated gradient algorithm: Convergence analysis and
applications. Applied Intelligence, 48, 2613-2627.
\end{thebibliography}
\end{document}